\documentclass[12pt]{amsart}
\usepackage[margin=1.2in]{geometry}
\usepackage{amsmath,amssymb,amsthm}
\usepackage{graphicx}
\usepackage{booktabs}
\usepackage{multirow}
\usepackage{float}
\usepackage{microtype}
\usepackage{url}      
\usepackage{placeins}

\newtheorem{theorem}{Theorem}

\newtheorem{proposition}[theorem]{Proposition}
\newtheorem{corollary}[theorem]{Corollary}
\newtheorem{definition}[theorem]{Definition}
\newtheorem{remark}[theorem]{Remark}

\title [Partition of Unity Neural Networks]{Partition of Unity Neural Networks for Interpretable Classification with Explicit Class Regions}

\author[A. Aldroubi]{Akram Aldroubi}
\address{Department of Mathematics\\Vanderbilt University}
\email{akram.aldroubi@vanderbilt.edu}
\begin{document}

\begin{abstract}

We introduce \emph{Partition of Unity Neural Networks} (PUNNs), a
neural-network architecture for multiclass classification based on the
classical mathematical notion of a partition of unity. The starting
point is the observation that the characteristic functions of ideal
class regions form a partition of unity. PUNNs replace these
discontinuous indicators by learned continuous functions
\[
h_1,\ldots,h_C:\mathcal X\to[0,1]
\]
whose sum is identically one and whose values are interpreted directly
as class probabilities.

The partition functions are generated through a recursive family of
input-dependent gates. This construction guarantees nonnegative class
probabilities summing to one without a separate normalization layer
such as softmax, while providing an explicit ordered factorization of
each probability in terms of the gate values. The resulting gate trace
gives an interpretable representation of how individual class
probabilities are formed. The framework also allows multiple partition
functions to represent a single class and admits both neural-network
and geometry-informed realizations of the gates.

We prove that PUNNs are dense in the space of continuous probability
maps from compact subsets of $\mathbb R^d$ into the probability
simplex. Thus, the recursive partition-of-unity structure retains
universal approximation of continuous probabilistic classifiers,
including maps whose components may vanish.

Numerical experiments on synthetic datasets, MNIST, and CIFAR-100
illustrate the learned partitions, the effect of class ordering, and
the use of multiple partition components per class. We also develop shape-informed gates that incorporate geometric information directly; when the chosen geometry is well matched to the class regions, these models achieve comparable accuracy with substantially fewer trainable parameters.
\end{abstract}

\maketitle

\section{Introduction}

Suppose that a classification problem has $C$ possible classes and that
each observation is represented by a feature vector
\[
x\in\mathbb R^d.
\]
The goal of a classifier is to determine which of the $C$ classes
$x$ belongs to.  More generally, to associate
with $x$ a probability vector
\[
p(x)
=
\bigl(p_1(x),\ldots,p_C(x)\bigr),
\]
where
\[
p_i(x)=P(Y=i\mid x),
\qquad i=1,\ldots,C,
\]
with
\[
p_i(x)\geq 0
\qquad\text{and}\qquad
\sum_{i=1}^C p_i(x)=1.
\]
The predicted class for $x$ is then typically taken to be
\[
\widehat y(x)
=
\operatorname*{arg\,max}_{1\leq i\leq C} p_i(x).
\]

A standard approach is to use a multilayer perceptron (MLP) \cite{goodfellow2016deep}. The input
$x$ is passed through a sequence of affine transformations followed by
nonlinear activation functions. Setting
\[
u^{(0)}(x)=x,
\]
the hidden layers are defined recursively by
\[
u^{(\ell)}(x)
=
\sigma_\ell\!\left(
W_\ell u^{(\ell-1)}(x)+b_\ell
\right),
\qquad
\ell=1,\ldots,L,
\]
where $W_\ell$ is the weight matrix and $b_\ell$ is the bias vector of
the $\ell$th layer, of dimensions compatible with the input and output
dimensions of that layer. The activation function $\sigma_\ell$ is
applied componentwise. A common choice is the rectified linear unit
(ReLU),
\[
\sigma_{\mathrm{ReLU}}(t)
=
\max\{0,t\}.
\]

The final affine layer produces the vector
\[
z(x)
=
W_{L+1}u^{(L)}(x)+b_{L+1}
=
\bigl(z_1(x),\ldots,z_C(x)\bigr)
\in\mathbb R^C,
\]
where $W_{L+1}$ and $b_{L+1}$ are the weight matrix and bias vector of
the output layer. The components $z_i(x)$ are called the \emph{logits}.
They may be viewed as class scores, but they are not themselves
probabilities: they need not be nonnegative and do not in general sum
to one.

To obtain class probabilities, the logits are typically passed through
the softmax map \cite{bishop2006pattern},
\[
P(Y=i\mid x)
=
\frac{e^{z_i(x)}}{\sum_{j=1}^C e^{z_j(x)}},
\qquad i=1,\ldots,C.
\]
The resulting quantities are nonnegative and satisfy
\[
\sum_{i=1}^C P(Y=i\mid x)=1,
\]
so they define a probability distribution over the $C$ classes. The
predicted class is then the one having the largest probability.

Given labeled training samples
\[
\{(x_n,y_n)\}_{n=1}^N,
\qquad
y_n\in\{1,\ldots,C\},
\]
the entries of the weight matrices $W_\ell$ and bias vectors $b_\ell$
are unknown parameters of the model. These parameters are learned from
the training data by minimizing a loss function. A standard choice is
the negative log-likelihood, equivalently the cross-entropy loss,
\[
\mathcal L
=
-\frac{1}{N}
\sum_{n=1}^N
\log P(Y=y_n\mid x_n).
\]
For the softmax classifier, this becomes
\[
\mathcal L
=
-\frac{1}{N}
\sum_{n=1}^N
\log
\left(
\frac{e^{z_{y_n}(x_n)}}
{\sum_{j=1}^C e^{z_j(x_n)}}
\right).
\]
The weight matrices $W_\ell$ and bias vectors $b_\ell$ are then chosen to minimize $\mathcal L$ over the training data.

Thus, in the standard construction, the passage from an input to a
classification may be summarized schematically as
\[
x
\longmapsto
\text{neural network}
\longmapsto
z(x)
\longmapsto
\text{softmax probabilities}
\longmapsto
\text{predicted class}.
\]
There is, however, another natural way to view the classification
problem. Ideally, the feature space is divided into regions
\[
\Omega_1,\ldots,\Omega_C
\]
corresponding to the $C$ classes. If these regions form a partition of
the feature space, then their characteristic functions
\[
\chi_{\Omega_1},\ldots,\chi_{\Omega_C}
\]
satisfy
\[
\chi_{\Omega_i}(x)\geq 0
\]
and
\[
\sum_{i=1}^C \chi_{\Omega_i}(x)=1
\]
at every point $x$. Thus, the characteristic functions themselves form
a partition of unity and provide an ideal deterministic classifier:
\[
\chi_{\Omega_i}(x)
=
\begin{cases}
1, & x\in\Omega_i,\\
0, & x\notin\Omega_i.
\end{cases}
\]

For learning from finite and possibly noisy data, such characteristic
functions are too rigid. The class regions are not known in advance,
their boundaries must be inferred from data, and a sample near a class
boundary may not naturally have probability exactly zero or one.
Moreover, characteristic functions are discontinuous at the boundaries
of the regions. This suggests approximating the characteristic functions of the class
regions by smooth functions that retain their essential behavior: each
function should be close to one on the region associated with its class,
close to zero away from that region, and the collection should sum to
one. Thus, we seek functions
\[
h_1,\ldots,h_C:\mathcal X\to[0,1]
\]
such that
\[
\sum_{i=1}^C h_i(x)=1.
\]
Such a family is precisely a partition of unity, a standard construction
in mathematics for replacing hard decompositions of a space by smooth
overlapping functions subordinate to those regions.

In the classification setting, the functions $h_i$ may therefore be
viewed as smooth approximations to the characteristic functions
$\chi_{\Omega_i}$, and at the same time interpreted probabilistically as
\[
P(Y=i\mid x)=h_i(x).
\]
This leads to the viewpoint that classification can be formulated as
the problem of learning a smooth partition of unity adapted to the
geometry of the feature space. We
introduce \emph{Partition of Unity Neural Networks} (PUNNs), in which
the class-probability functions themselves form a learned partition of
unity. Rather than first producing an unconstrained vector of logits
and then normalizing it by softmax, PUNN constructs the probability
functions directly.

\subsection{Overview of the PUNN Architecture}

The PUNN architecture constructs the class-probability functions from
a collection of input-dependent gates
\[
q_1,\ldots,q_{C-1}:\mathcal X\to[0,1].
\]
These gates are combined recursively to produce functions
\[
h_1,\ldots,h_C:\mathcal X\to[0,1]
\]
that satisfy
\[
\sum_{i=1}^C h_i(x)=1
\]
for every $x\in\mathcal X$. The functions $h_i$ are interpreted
directly as the class probabilities.

Thus, the PUNN classification pipeline may be summarized schematically as
\[
x
\longmapsto
\bigl(q_1(x),\ldots,q_{C-1}(x)\bigr)
\longmapsto
\bigl(h_1(x),\ldots,h_C(x)\bigr)
\longmapsto
\text{predicted class}.
\]

The recursive construction is defined precisely in
Section~\ref{sec:architecture}. It guarantees nonnegative outputs that
sum to one by construction and yields an explicit factorization of the
class probabilities through the gate values.

The gates may be realized by separate neural networks, by a shared
multi-output neural network, or by low-dimensional geometric
parameterizations. We also show that the resulting PUNN class is dense
in the space of continuous probability maps from compact domains into
the probability simplex.

\subsection{Contributions}
The main contributions of this work are the following:
\begin{itemize}

\item \textbf{PUNN architecture.}
We introduce a neural-network architecture for multiclass
classification in which the class probabilities are represented
directly as the components of a learned partition of unity. The
partition is generated through a recursive gate construction, so that
the outputs are nonnegative and sum to one by construction, without a
softmax normalization layer.

\item \textbf{Recursive parameterization.}
We introduce a recursive gate construction that produces such a
partition of unity automatically. The construction admits separate
neural-network gates, shared multi-output neural realizations, and
geometry-informed gate families.

\item \textbf{Approximation theory.}
We prove that the resulting PUNN class is dense in the space of
continuous probability maps from compact domains into the probability
simplex.

\item \textbf{Interpretable gate trace.}
The recursive construction yields an explicit ordered factorization of
each class probability into gate values and complementary factors.
For each input, this produces a gate-level trace showing how the
corresponding class probabilities are formed. The trace can be used to
analyze individual predictions and identify which gates most strongly
contribute to a classification decision.

\item \textbf{Numerical illustrations.}
We study the construction on synthetic datasets, MNIST, and CIFAR-100.
The experiments illustrate the learned partition functions, the effect
of class ordering, extensions with multiple partition functions per
class, and geometry-informed parameterizations.
\end{itemize}

\section{Related Work}
\label{sec:related}

\subsection{Partitions of Unity in Approximation and Learning}

Partitions of unity are classical tools in analysis, topology, and
differential geometry, where collections of nonnegative functions that
sum to one are used to localize and combine constructions
\cite{Lee2013}. They also play an important role in approximation
theory and numerical analysis, including partition-of-unity methods for
combining local approximants \cite{BabuskaMelenk1997}. Related
partition constructions also arise in harmonic analysis, for example
through families of functions whose dilates or translates resolve the
identity.

More recently, partitions of unity have appeared in neural-network
approximation theory, although for purposes different from the
classification framework considered here. Shaham, Cloninger, and
Coifman \cite{shaham2018provable} use a partition of unity on a
low-dimensional manifold to assemble local neural-network
approximations into a global approximation. Lee et al.\
\cite{lee2021pounets} introduce Partition of Unity Networks, in which a
learned partition of the input domain is combined with local polynomial
approximation spaces. More recently, Shi and Liao
\cite{shi2026transformers} use a softmax-generated partition of unity
to combine local approximations in their approximation theory for
Transformer networks. In these works, the partition of unity is used
primarily to localize or aggregate approximations. In PUNNs, by
contrast, the components of the learned partition of unity are
themselves the class-probability functions.

\subsection{Related Neural Classification Architectures}

Several neural architectures produce structured outputs through learned
gating or hierarchical mechanisms. Examples include mixture-of-experts
models, gating networks, and hierarchical classifiers
\cite{jacobs1991adaptive,jordan1994hierarchical,
shazeer2017outrageously,kontschieder2015deep}.

PUNNs are related to these constructions at a broad structural level,
but the role of the gates is different. In mixture-of-experts models, a
gating network typically produces weights used to combine the outputs
of several expert networks. In PUNNs, the gates do not select or combine
expert predictions. Instead, they parameterize the components of a
partition of unity, and these components are themselves the class
probabilities.

Thus, the partition of unity in a PUNN is the output structure of the
classifier rather than a routing mechanism between learned experts.
The specific recursive parameterization of this partition is developed
in the next section.
\section{The PUNN Architecture}
\label{sec:architecture}
\subsection{Recursive Partition-of-Unity Construction}
\label{subsec:recursive_pou}

Let $\mathcal X$ denote the feature space, and let
\[
q_i:\mathcal X\to[0,1],
\qquad i=1,\ldots,k-1,
\]
be a collection of gate functions. These gates recursively parameterize
$k$ partition functions
\[
h_1,\ldots,h_k:\mathcal X\to[0,1],
\]
defined by
\begin{align}
h_1(x)
&=
q_1(x),
\label{eq:pou_h1}\\
h_i(x)
&=
\left(
\prod_{j=1}^{i-1}(1-q_j(x))
\right)q_i(x),
\qquad i=2,\ldots,k-1,
\label{eq:pou_hi}\\
h_k(x)
&=
\prod_{j=1}^{k-1}(1-q_j(x)).
\label{eq:pou_hk}
\end{align}

For each fixed $x\in\mathcal X$, the values
$h_1(x),\ldots,h_k(x)$ are assembled algebraically from the gate
values $q_1(x),\ldots,q_{k-1}(x)$. For $i<k$, the component $h_i(x)$
is formed from the output of gate $i$ together with the complementary
outputs of the preceding gates, while the final component $h_k(x)$ is
formed from the complementary outputs of all $k-1$ gates. Thus, the
ordering determines the factorization of the partition components; it
does not imply that the gates are trained one at a time.

\begin{proposition}[Partition-of-Unity Property]
\label{prop:partition}
Let $k\geq 2$. For every collection of gates
\[
q_i:\mathcal X\to[0,1],
\qquad i=1,\ldots,k-1,
\]
the functions defined by
\eqref{eq:pou_h1}--\eqref{eq:pou_hk} satisfy
\[
h_i(x)\geq 0,
\qquad i=1,\ldots,k,
\]
and
\[
\sum_{i=1}^k h_i(x)=1
\]
for every $x\in\mathcal X$.
\end{proposition}

\begin{proof}
Since $0\leq q_i(x)\leq 1$, each $h_i(x)$ is nonnegative. Moreover,
\begin{align*}
\sum_{i=1}^k h_i(x)
&=
\sum_{i=1}^{k-1}
\left(
\prod_{j=1}^{i-1}(1-q_j(x))
\right)q_i(x)
+
\prod_{j=1}^{k-1}(1-q_j(x))\\
&=
1.
\end{align*}
Indeed, the sum telescopes according to
\[
\sum_{i=1}^{m}
\left(
\prod_{j=1}^{i-1}(1-q_j(x))
\right)q_i(x)
=
1-\prod_{j=1}^{m}(1-q_j(x)),
\qquad m=1,\ldots,k-1.
\]
Taking $m=k-1$ gives the result.
\end{proof}

\begin{remark}[Regularity]
If the gates $q_1,\ldots,q_{k-1}$ are continuous, then the resulting
partition functions $h_1,\ldots,h_k$ are continuous. Similarly, if the
gates are differentiable or smooth, then the partition functions have
the same regularity.
\end{remark}

\begin{remark}
\label{rmk:gate_realization_independence}
The partition-of-unity property depends only on the recursive
construction and not on the particular realization of the gates. Thus,
any collection of functions
\[
q_i:\mathcal X\to[0,1],
\qquad i=1,\ldots,k-1,
\]
produces a valid partition of unity through
\eqref{eq:pou_h1}--\eqref{eq:pou_hk}. Neural-network and geometry-informed realizations of the gates are introduced in Section~\ref{sec:realizations}.
\end{remark}

\subsection{PUNN for Classification}
\label{subsec:punn_classification}

For a classification problem with $C$ classes, the basic construction
takes $k=C$ partition functions and associates one partition function
with each class. The generalization to $k>C$ is discussed in
Remark~\ref{rmk:multimodal} below. The conditional class probabilities
are defined directly by
\[
P(Y=i\mid x)
=
h_i(x),
\qquad i=1,\ldots,C.
\]

By Proposition~\ref{prop:partition}, these class probabilities are
nonnegative and satisfy
\[
\sum_{i=1}^C P(Y=i\mid x)
=
1
\]
for every $x\in\mathcal X$. Thus, the partition functions themselves
form the class-probability map, and no additional normalization such
as softmax is required.

The predicted class is
\[
\widehat y(x)
=
\operatorname*{arg\,max}_{1\leq i\leq C} h_i(x).
\]

As in the standard probabilistic classifier described in the
Introduction, training is based on the negative log-likelihood. Since
the PUNN class probabilities are given directly by
$P(Y=i\mid x)=h_i(x)$, the loss becomes
\[
\mathcal L
=
-\frac{1}{N}
\sum_{n=1}^N
\log h_{y_n}(x_n).
\]
All gate parameters are optimized jointly through $\mathcal L$. Thus,
the distinction from the standard softmax classifier lies in the
parameterization of the class probabilities, not in the choice of loss
function. The class ordering determines the algebraic factorization of
the probabilities, not the order in which the gates are trained.

\begin{remark}[Extension to Multimodal Classes]
\label{rmk:multimodal}
The number of partition functions need not equal the number of classes.
If a class occupies several disjoint regions in the feature space, one may take
$k>C$ and assign several partition functions to the same class. Let
\[
\pi:\{1,\ldots,k\}\to\{1,\ldots,C\}
\]
map each partition index to a class label. The class probabilities are
then defined by
\[
P(Y=c\mid x)
=
\sum_{i:\,\pi(i)=c} h_i(x),
\qquad c=1,\ldots,C.
\]
Since the functions $h_i$ are nonnegative and sum to one, the
aggregated class probabilities are also nonnegative and sum to one.

In this case, the negative log-likelihood uses the aggregated
probability of the true class:
\[
\mathcal L
=
-\frac{1}{N}
\sum_{n=1}^N
\log\left(
\sum_{i:\,\pi(i)=y_n} h_i(x_n)
\right).
\]
This allows distinct partition functions to represent different
regions or modes associated with the same class.
\end{remark} 

\section{Theoretical Analysis}
\label{sec:theory}

\subsection{Approximation of Probability Maps}

Universal approximation theorems establish that standard neural
networks can approximate continuous real-valued functions on compact
domains to arbitrary precision. Classical results include
\cite{cybenko1989approximation} for sigmoidal networks and
\cite{hornik1989multilayer} for more general activation functions.

Our analysis concerns approximation within a structured function class
appropriate for probabilistic classification. Rather than approximating
arbitrary real-valued functions, we consider continuous
\emph{probability maps}, namely, continuous functions taking values in
the probability simplex. We first establish the approximation result for
maps taking values in the relative interior of the simplex, and then
extend it to arbitrary continuous probability maps by approximation.
Thus, the recursive partition-of-unity parameterization retains
universal approximation for continuous probability maps taking values
in the entire probability simplex.

\subsection{Density of PUNNs in the Space of Continuous Probability Maps}

\begin{definition}[$(k-1)$-Simplex]
The $(k-1)$-dimensional probability simplex is defined by
\[
\Delta^{k-1}
=
\left\{
(p_1,\dots,p_k)\in\mathbb{R}^k
:
p_i\geq 0 \text{ for all } i,
\quad
\sum_{i=1}^k p_i=1
\right\}.
\]
Its relative interior is
\[
\operatorname{ri}(\Delta^{k-1})
=
\left\{
(p_1,\dots,p_k)\in\mathbb{R}^k
:
p_i>0 \text{ for all } i,
\quad
\sum_{i=1}^k p_i=1
\right\}.
\]
\end{definition}

\begin{definition}[Probability Map]
Let $K\subset\mathbb{R}^d$ be compact. A function
\[
p:K\to\Delta^{k-1}
\]
is called a \emph{probability map}. If each coordinate function $p_i$
is continuous, then $p$ is called a \emph{continuous probability map}.
\end{definition}

\begin{remark}
Classical universal approximation results can also be used to
approximate simplex-valued targets with a neural network followed by a
softmax output layer. The question addressed here is therefore not
whether neural networks can approximate probability maps, but whether
the recursive partition-of-unity parameterization used by PUNNs
retains this approximation property. In this parameterization, the
class probabilities are assembled recursively from input-dependent
gates rather than obtained by globally normalizing a vector of logits.
Theorem~\ref{thm:punn_density} establishes the approximation result for
continuous maps into $\operatorname{ri}(\Delta^{k-1})$, and
Corollary~\ref{cor:punn_density_simplex} extends it to continuous maps
into the full simplex $\Delta^{k-1}$.
\end{remark}

\begin{theorem}
\label{thm:punn_density}
Let $K\subset\mathbb{R}^d$ be compact, and let
\[
p:K\to\operatorname{ri}(\Delta^{k-1})
\]
be a continuous probability map. Let
\[
g:\mathbb{R}\to(0,1)
\]
be a continuous strictly monotone bijection, and let
$\mathcal{N}\subset C(K)$ be a class of feedforward neural networks that is
dense in $C(K)$ with respect to the uniform norm.

For $\theta_1,\ldots,\theta_{k-1}\in\mathcal{N}$, define the PUNN partition
functions by
\begin{align*}
h_1(x)
&=
g(\theta_1(x)),\\
h_i(x)
&=
\left(
\prod_{j=1}^{i-1}
\bigl(1-g(\theta_j(x))\bigr)
\right)
g(\theta_i(x)),
\qquad i=2,\ldots,k-1,\\
h_k(x)
&=
\prod_{j=1}^{k-1}
\bigl(1-g(\theta_j(x))\bigr).
\end{align*}

Then, for every $\varepsilon>0$, there exist
$\theta_1,\ldots,\theta_{k-1}\in\mathcal{N}$ such that
\[
\max_{1\leq i\leq k}
\|h_i-p_i\|_{C(K)}
<
\varepsilon.
\]
\end{theorem}
\begin{proof}
Let
\[
p=(p_1,\dots,p_k):K\to\operatorname{ri}(\Delta^{k-1})
\]
be a continuous probability map.
In particular,
\[
p_i(x) > 0
\quad \text{and} \quad
\sum_{i=1}^k p_i(x) = 1
\quad \text{for all } x \in K.
\]

\medskip
\noindent
Define continuous functions $\gamma_1,\dots,\gamma_{k-1} : K \to (0,1)$ recursively by
\begin{align*}
\gamma_1(x) &= p_1(x), \\
\gamma_i(x) &= \frac{p_i(x)}{\sum_{j=i}^k p_j(x)},
\quad i = 2,\dots,k-1.
\end{align*}
Since $p$ is continuous and $\sum_{j=i}^k p_j(x) > 0$ for all $x$, each $\gamma_i$ is continuous and takes values in $(0,1)$.

\medskip
\noindent
Let $g : \mathbb{R} \to (0,1)$ be a continuous strictly monotone bijection, and define
\[
q_i(x) := g(\phi_i(x)),
\quad \text{where } \phi_i(x) := g^{-1}(\gamma_i(x)).
\]

\medskip
\noindent
Since $g$ is strictly monotone and continuous, it follows that $g^{-1}$ is also monotone and continuous and that $\phi_i := g^{-1}\circ \gamma_i$ is continuous on $K$.

\medskip
\noindent
Define the PUNN partition functions $h_1,\dots,h_k$ by
\begin{align*}
h_1(x) &= q_1(x), \\
h_i(x) &= \left( \prod_{j=1}^{i-1} (1 - q_j(x)) \right) q_i(x),
\quad i = 2,\dots,k-1, \\
h_k(x) &= \prod_{j=1}^{k-1} (1 - q_j(x)).
\end{align*}

\medskip
\noindent
A direct computation then shows that
\[
h_i(x) = p_i(x), \qquad i = 1,\dots,k.
\]

\medskip
\noindent
By the classical universal approximation theorem for feedforward neural networks
(e.g.\ \cite{cybenko1989approximation,hornik1989multilayer}),
for every $\delta > 0$ there exist neural networks
$\theta_i : K \to \mathbb{R}$ such that
\begin{equation}
\label{eq:delta}
\|\theta_i - \phi_i\|_{C(K)} < \delta,
\quad i = 1,\dots,k-1.
\end{equation}

\medskip
\noindent
Since $\phi_i$ is continuous on the compact set $K$, its range is contained in a
compact interval $[-M,M] \subset \mathbb{R}$. Choosing $\delta$ sufficiently small in
\eqref{eq:delta}, we may assume that $\theta_i(K) \subset [-M-1,M+1]$ for all $i$.
Since $g$ is continuous, it is uniformly continuous on the compact interval
$[-M-1,M+1]$. Therefore, for every $\eta>0$ there exists $\delta>0$ such that
\[
\|g(\theta_i)-g(\phi_i)\|_{C(K)} < \eta
\quad \text{for all } i.
\]

\medskip
\noindent
Let $\widetilde h_1,\dots,\widetilde h_k$ denote the partition functions obtained
by replacing $\phi_i$ with $\theta_i$ in the above construction. Since the ranges
of $\widetilde h_i$ and $h_i$ are contained in $[0,1]$, and since each $h_i$ and
$\widetilde h_i$ is a finite product of the functions $q_j$, $1-q_j$,
$g(\theta_j)$, and $1-g(\theta_j)$, all of which take values in $[0,1]$,
choosing $\eta>0$ sufficiently small depending on $k$ and $\varepsilon$, and
possibly reducing $\delta$ in \eqref{eq:delta}, we obtain
\[
\max_{1\le i\le k} \|\widetilde h_i - h_i\|_{C(K)} < \varepsilon.
\]

\medskip
\noindent
To make this explicit, consider the second partition function:
\[
\widetilde h_2
=
\bigl(1-g(\theta_1)\bigr)g(\theta_2)
\quad\text{and}\quad
h_2
=
\bigl(1-g(\phi_1)\bigr)g(\phi_2).
\]
We write
\begin{align*}
\widetilde h_2-h_2
&=
\bigl(1-g(\theta_1)\bigr)g(\theta_2)
-
\bigl(1-g(\phi_1)\bigr)g(\phi_2)
\\
&=
\bigl(g(\phi_1)-g(\theta_1)\bigr)g(\theta_2)
+
\bigl(1-g(\phi_1)\bigr)
\bigl(g(\theta_2)-g(\phi_2)\bigr).
\end{align*}
Taking supremum norms and using that all factors take values in $[0,1]$,
we obtain
\begin{align*}
\|\widetilde h_2-h_2\|_{C(K)}
&\le
\|g(\theta_1)-g(\phi_1)\|_{C(K)}
+
\|g(\theta_2)-g(\phi_2)\|_{C(K)}.
\end{align*}
The same telescoping-product argument applies to the remaining partition
functions.

\medskip
\noindent
Since $h_i=p_i$, this yields
\[
\max_{1\le i\le k} \|\widetilde h_i - p_i\|_{C(K)} < \varepsilon,
\]
which completes the proof.
\end{proof}

\begin{corollary}
\label{cor:punn_density_simplex}
Under the assumptions of Theorem~\ref{thm:punn_density}, the PUNN
partition functions are dense in the space of continuous probability
maps
\[
p:K\to\Delta^{k-1}.
\]
That is, for every continuous probability map
\[
p=(p_1,\ldots,p_k):K\to\Delta^{k-1}
\]
and every $\varepsilon>0$, there exist
$\theta_1,\ldots,\theta_{k-1}\in\mathcal N$ such that the corresponding
PUNN partition functions $h_1,\ldots,h_k$ satisfy
\[
\max_{1\leq i\leq k}
\|h_i-p_i\|_{C(K)}
<
\varepsilon.
\]
\end{corollary}

\begin{proof}
Let $\varepsilon>0$. For $0<\delta<1$, define
\[
p_i^{(\delta)}(x)
=
(1-\delta)p_i(x)+\frac{\delta}{k},
\qquad i=1,\ldots,k.
\]
Then
\[
p^{(\delta)}
=
\bigl(p_1^{(\delta)},\ldots,p_k^{(\delta)}\bigr)
:
K\to\operatorname{ri}(\Delta^{k-1})
\]
is continuous. Indeed,
\[
p_i^{(\delta)}(x)
\geq
\frac{\delta}{k}
>
0
\]
for every $i$ and every $x\in K$, and
\[
\sum_{i=1}^k p_i^{(\delta)}(x)
=
(1-\delta)\sum_{i=1}^k p_i(x)+\delta
=
1.
\]
Moreover,
\[
\|p_i^{(\delta)}-p_i\|_{C(K)}
=
\delta
\left\|
\frac{1}{k}-p_i
\right\|_{C(K)}
\leq
\delta,
\]
since $0\leq p_i(x)\leq 1$.

Choose
\[
0<\delta<\min\left\{1,\frac{\varepsilon}{2}\right\}.
\]
By Theorem~\ref{thm:punn_density}, there exist
$\theta_1,\ldots,\theta_{k-1}\in\mathcal N$ such that the corresponding
PUNN partition functions satisfy
\[
\max_{1\leq i\leq k}
\|h_i-p_i^{(\delta)}\|_{C(K)}
<
\frac{\varepsilon}{2}.
\]
Therefore, for every $i$,
\[
\|h_i-p_i\|_{C(K)}
\leq
\|h_i-p_i^{(\delta)}\|_{C(K)}
+
\|p_i^{(\delta)}-p_i\|_{C(K)}
<
\frac{\varepsilon}{2}+\delta
<
\varepsilon.
\]
\end{proof}

\begin{remark}[Multi-output neural realization]
\label{rmk:multioutput_density}
The same density result holds for the multi-output neural realization.
Indeed, in the proof of Theorem~\ref{thm:punn_density}, the target gate
arguments define a continuous vector-valued function
\[
\Phi
=
(\phi_1,\ldots,\phi_{k-1})
:
K\to\mathbb R^{k-1}.
\]
If the class of multi-output neural networks is dense in
\[
C(K,\mathbb R^{k-1})
\]
with respect to the uniform norm, then a single network
\[
\Theta
=
(\theta_1,\ldots,\theta_{k-1})
\]
can approximate $\Phi$ uniformly. The remainder of the proof is
unchanged. Thus, the approximation results apply to both the fully
separate and multi-output neural realizations, provided the corresponding
network class has the required universal approximation property. This
does not assert universal approximation for an arbitrary fixed
multi-output architecture of prescribed finite size.
\end{remark}

\section{Realizations of the Partition}
\label{sec:realizations}

\subsection{Neural-Network Realizations}
\label{subsec:neural_realization}

A convenient way to parameterize the gates introduced above is to write
\[
q_i(x)=g(\theta_i(x)),
\qquad i=1,\ldots,k-1,
\]
where
\[
\theta_i:\mathcal X\to\mathbb R
\]
is a real-valued function, called the \emph{gate argument}, and
\[
g:\mathbb R\to[0,1]
\]
is a scalar activation function. Thus, the gate function $q_i$ is
obtained by applying the activation $g$ to the gate argument
$\theta_i$.

We consider two neural-network realizations of the gate arguments.
In the first, each function $\theta_i$ is parameterized by a separate
feedforward neural network. In the second, a single multi-output neural
network produces all gate arguments from a shared learned
representation. In the fully separate realization, one may write
\[
\theta_i(x)=\operatorname{MLP}_i(x),
\qquad i=1,\ldots,k-1,
\]
and therefore
\[
q_i(x)
=
g\bigl(\operatorname{MLP}_i(x)\bigr).
\]
Here, each gate argument $\theta_i$ is produced by its own neural
network.

A second neural realization uses a single multi-output network to
produce all gate arguments simultaneously. Specifically, let
\[
\Theta:\mathcal X\to\mathbb R^{k-1},
\qquad
\Theta(x)
=
\bigl(\theta_1(x),\ldots,\theta_{k-1}(x)\bigr).
\]
The gate outputs are then defined componentwise by
\[
q_i(x)=g\bigl(\theta_i(x)\bigr),
\qquad i=1,\ldots,k-1.
\]
In this realization, the functions $\theta_1,\ldots,\theta_{k-1}$
share the hidden representation learned by the single network
$\Theta$, while retaining distinct scalar outputs. This can
substantially reduce parameter duplication when the number of gates or
partition functions is large. The recursive partition-of-unity
construction is unchanged; only the parameterization of the gate
arguments differs.

In either neural realization, all trainable parameters are optimized
jointly through the loss function introduced in
Section~\ref{subsec:punn_classification}.

Because each gate has its own scalar argument $\theta_i$, different
gates may adapt to different geometric or statistical features of the
data. Feedforward neural networks provide flexible parameterizations of
these functions \cite{cybenko1989approximation,hornik1989multilayer}.
Under the assumptions of Section~\ref{sec:theory}, the resulting
partition-of-unity maps are dense in the space of continuous probability
maps into the simplex.

Natural choices for the scalar  activation function  include the sigmoid,
Gaussian, and compactly supported bump functions:
\[
g_{\mathrm{sig}}(t)
=
\frac{1}{1+e^{-t}},
\]
\[
g_{\mathrm{gauss}}(t)
=
e^{-t^2},
\]
and
\[
g_{\mathrm{bump}}(t)
=
\begin{cases}
\exp\left(1-\dfrac{1}{1-t^2}\right),
& |t|<1,\\[1ex]
0,
& |t|\geq 1.
\end{cases}
\]

The sigmoid gives a monotone transition across a learned level set of
$\theta_i$, whereas the Gaussian and bump activations are localized
around the level set $\theta_i(x)=0$. The bump activation additionally
vanishes when $|\theta_i(x)|\geq 1$.

More generally, different gates may use different activations:
\[
q_i(x)=g^{(i)}(\theta_i(x)),
\qquad
g^{(i)}:\mathbb R\to[0,1].
\]
Thus, monotone and localized gates may be combined within the same
model. The activation determines how the scalar argument is mapped into
$[0,1]$, while the choice of $\theta_i$ determines its variation over
the feature space. Regardless of these choices, the recursive
construction preserves the partition-of-unity property by
Proposition~\ref{prop:partition}.

\subsection{Geometry-Informed Realization}
\label{subsec:geometry_realization}

For the geometry-informed realizations below, we assume that
$\mathcal X\subseteq\mathbb R^d$. We retain the representation
\[
q_i(x)=g(\theta_i(x)),
\]
where $q_i$ is the gate function, $\theta_i$ is the gate argument, and
$g$ is the scalar activation function. In this realization, the gate
argument $\theta_i$ is parameterized by a low-dimensional geometric
model on the feature space rather than by a neural network. We first
illustrate this idea using radial and ellipsoidal parameterizations.
More general shape-informed parameterizations, including shells with
direction-dependent inner and outer radii, are developed in
Section~\ref{sec:shape_informed}.

For example, a radial gate argument may be defined by
\[
\theta_i(x)
=
s_i\bigl(r_i-\|x-c_i\|\bigr),
\]
where the center $c_i$, radius $r_i$, and scale parameter $s_i>0$ are
learned from the data. When $g$ is the sigmoid scalar activation
function, the resulting gate function
\[
q_i(x)=g(\theta_i(x))
\]
is close to one inside a soft spherical region and close to zero
outside.

Similarly, an ellipsoidal gate argument may be defined by
\[
\theta_i(x)
=
s_i\left(
1-(x-c_i)^\top A_i(x-c_i)
\right),
\]
where the center $c_i$, scale parameter $s_i>0$, and positive-definite
matrix $A_i$ are learned from the data. The matrix $A_i$ determines
the size, orientation, and principal axes of the corresponding
ellipsoidal level sets of $\theta_i$. After application of the scalar
activation function $g$, these level sets determine the geometry of the
gate function $q_i$.

In these geometry-informed realizations, the trainable parameters are
the geometric parameters defining the gate arguments rather than the
weights and biases of a multilayer perceptron. Such parameterizations
can substantially reduce the number of trainable parameters when the
chosen geometric family is well aligned with the class geometry. They
also make the geometry underlying the gate functions directly
interpretable through quantities such as centers, radii, orientations,
and principal axes.

\subsection{Ordered Gate Decomposition}
\label{subsec:hierarchical_structure}

For the basic construction with $k=C$, the recursive formulation gives
an ordered factorization of the class probabilities. For
$i=1,\ldots,C-1$,
\[
h_i(x)
=
\left(
\prod_{j=1}^{i-1}(1-q_j(x))
\right)q_i(x),
\]
while
\[
h_C(x)
=
\prod_{j=1}^{C-1}(1-q_j(x)).
\]

For a fixed input $x$, the factor $q_i(x)$ may be interpreted as the
acceptance factor associated with class $i$, while
$1-q_j(x)$, $j<i$, represents the complementary factor associated with
the preceding classes. Thus, each class probability is expressed
explicitly in terms of the gate values.

All gate functions are evaluated on the same input and their parameters are
learned jointly. Nevertheless, the factorization yields the gate trace
\[
\bigl(q_1(x),\ldots,q_{C-1}(x)\bigr),
\]
which records the quantities from which the class-probability vector is
formed.

This representation can be used to diagnose individual predictions. A
small probability for the true class may result either from a small
value of its associated gate function or from large values of preceding gates,
which produce small complementary factors. Section~\ref{sec:experiments}
illustrates this decomposition on MNIST.

Unlike softmax, which obtains class probabilities through a common
normalization of logits, the PUNN representation retains these
gate-level factors explicitly while the partition functions
$h_1,\ldots,h_C$ remain the class-probability outputs.



\subsection{Flexibility of Gate Design}

Theorem~\ref{thm:punn_density} and
Corollary~\ref{cor:punn_density_simplex} apply when the scalar activation function
\[
g:\mathbb{R}\to(0,1)
\]
is a continuous strictly monotone bijection and the gate arguments are
drawn from a neural-network class with the required universal
approximation property. In particular, these results apply when the scalar activation function is the sigmoid. They do not establish an analogous density result for nonmonotone activations such as the Gaussian or bump functions.

The recursive partition-of-unity construction itself is more general.
Different gates within the same PUNN may use different activation functions
and different parameterizations, allowing heterogeneous inductive biases
across classes. For example, geometry-informed gates may be used for classes
with simple geometric structure, while MLP-based gates may be used for classes
with more complex or irregular regions.

Each gate may be written as
\[
q_i(x)=g^{(i)}(\theta_i(x)),
\]
where
\[
g^{(i)}:\mathbb{R}\to[0,1]
\]
is the activation associated with the $i$th gate and $\theta_i$ is its
gate argument. The choice of activation and the choice of parameterization
play distinct roles: the activation determines how the scalar argument is
mapped into $[0,1]$, while the parameterization determines how that argument
varies over the feature space.

\textbf{Neural-network-based gates} may use either of the two
realizations described in Section~\ref{subsec:neural_realization}. In
the fully separate realization,
\[
\theta_i(x)=\operatorname{MLP}_i(x),
\]
whereas in the multi-output realization the gate arguments
$\theta_1(x),\ldots,\theta_{k-1}(x)$ are produced jointly by a single
network. In our experiments, the naming convention
\textbf{PUNN-[Activation]} specifies the scalar activation function
independently of which neural realization is used. For example,
PUNN-Sigma uses sigmoid activations with neural-network-parameterized
gate arguments.

\textbf{Geometry-informed gates} encode geometric priors directly into
$\theta_i$. Examples include spherical shells, ellipsoids, and
direction-dependent radii parameterized by Fourier series or spherical
harmonics. Such parameterizations can substantially reduce the number of
trainable parameters when the assumed geometry is appropriate.
Section~\ref{sec:shape_informed} develops these constructions in detail.

\begin{remark}[Partition-of-Unity Preservation]
The assumptions on the scalar activation function and gate arguments in
Theorem~\ref{thm:punn_density} are needed for the approximation result,
not for the partition-of-unity property itself. For any gate functions
\[
q_i:\mathcal X\to[0,1],
\qquad i=1,\ldots,k-1,
\]
the recursive construction satisfies
\[
\sum_{i=1}^k h_i(x)=1
\]
for every $x\in\mathcal X$, independently of how the gates are
parameterized.
\end{remark}

\section{Shape-Informed Gate Parameterizations}
\label{sec:shape_informed}

The geometry-informed realizations introduced in
Section~\ref{subsec:geometry_realization} can be extended to more
general class regions by allowing the relevant geometric parameters to
depend on direction. We describe here a family of gate functions based
on direction-dependent inner and outer radii. This construction includes
spherical and star-shaped regions as special cases and admits
low-dimensional parameterizations using Fourier series or spherical
harmonics.

\subsection{Direction-Dependent Shell Parametrization}

We introduce a unified framework for shape-informed gates using
direction-dependent radii in spherical coordinates. For each gate, the
geometry is described relative to a center $c_i\in\mathbb{R}^d$ and by
inner and outer radius functions on the unit sphere $S^{d-1}$.

\begin{definition}[Direction-Dependent Shell Gate]
Let $c_i\in\mathbb{R}^d$, and let
\[
r_{i,\mathrm{in}},r_{i,\mathrm{out}}
:
S^{d-1}\to\mathbb{R}_{\geq 0}
\]
be direction-dependent radius functions satisfying
\[
r_{i,\mathrm{in}}(\hat n)
<
r_{i,\mathrm{out}}(\hat n)
\qquad
\text{for every }\hat n\in S^{d-1}.
\]
For $x\neq c_i$, define
\[
\hat n_i(x)
=
\frac{x-c_i}{\|x-c_i\|}
\]
and the normalized radial coordinate
\[
t_i(x)
=
\frac{
\|x-c_i\|-r_{i,\mathrm{in}}(\hat n_i(x))
}{
r_{i,\mathrm{out}}(\hat n_i(x))
-
r_{i,\mathrm{in}}(\hat n_i(x))
}.
\]
The corresponding gate function is
\[
q_i(x)
=
g\!\bigl(t_i(x)\bigr),
\]
where
\[
g:\mathbb{R}\to[0,1]
\]
is a smooth scalar activation function.
\end{definition}
At $x=c_i$, where the direction $\hat n_i(x)$ is undefined, the gate
value may be specified separately according to the intended geometry.

The level sets
\[
t_i(x)=0
\qquad\text{and}\qquad
t_i(x)=1
\]
describe the inner and outer boundaries of the shell. Different choices
of the scalar activation function $g$ determine the profile of the gate
between and across these boundaries. In particular, one may choose a
smooth function $g:\mathbb R\to[0,1]$ supported on $[0,1]$, so that
$q_i(x)$ vanishes outside the shell region
\[
\Omega_i
=
\left\{
c_i+r\hat n:
\hat n\in S^{d-1},\;
r_{i,\mathrm{in}}(\hat n)
\leq r
\leq
r_{i,\mathrm{out}}(\hat n)
\right\}.
\]

\subsection{Special Cases}

The direction-dependent shell construction includes several familiar
geometric families.

\paragraph{Spherical shells and balls.}
If the inner and outer radius functions are constant,
\[
r_{i,\mathrm{in}}(\hat n)=r_{i,\mathrm{in}},
\qquad
r_{i,\mathrm{out}}(\hat n)=r_{i,\mathrm{out}},
\]
the resulting geometry is a spherical shell centered at $c_i$.
Ball-type gates are obtained by using a single outer radial boundary,
for example through a gate argument of the form
\[
\theta_i(x)
=
s_i\bigl(r_i-\|x-c_i\|\bigr),
\]
as in Section~\ref{subsec:geometry_realization}.

\paragraph{Ellipsoidal regions.}
Ellipsoidal geometries may be obtained by allowing the radial boundary
to vary with direction according to a quadratic form. This recovers the
ellipsoidal gate arguments described in
Section~\ref{subsec:geometry_realization}.

\paragraph{Star-shaped regions in $\mathbb R^2$.}
In two dimensions, direction-dependent radii may be represented by
truncated Fourier series. For example,
\[
r_i(\theta)
=
a_{i,0}
+
\sum_{m=1}^{M}
\left(
a_{i,m}\cos(m\theta)
+
b_{i,m}\sin(m\theta)
\right).
\]
Using separate expansions for the inner and outer radii gives
star-shaped shell regions of increasing geometric complexity as
$M$ increases.

\paragraph{Higher-dimensional star-shaped regions.}
In dimensions $d\geq 3$, direction-dependent radii may instead be
expanded in spherical harmonics,
\[
r_i(\hat n)
=
\sum_{\ell=0}^{L}
\sum_m
a_{i,\ell,m}Y_\ell^m(\hat n).
\]
The truncation degree $L$ controls the geometric complexity of the
resulting family.

\section{Experiments}
\label{sec:experiments}

We illustrate the PUNN construction on synthetic and standard
classification datasets. The experiments are intended to demonstrate
the learned partition-of-unity structure, the effect of class ordering,
the use of shared gate representations, and the extension to multiple
partition functions per class. We also evaluate the shape-informed gate
parameterizations developed in Section~\ref{sec:shape_informed}.

\subsection{Synthetic Datasets}
\label{sec:synthetic}

We begin our empirical evaluation with synthetic two-dimensional
classification problems. These experiments serve two purposes: first, to
illustrate the ability of different PUNN realizations to learn nonlinear class regions; and second, to visualize the learned partition functions, which are most transparent in low-dimensional feature spaces.

\subsubsection{Datasets}

We consider four binary classification datasets of varying geometric complexity:
\textbf{Moons}, consisting of two interleaving half-circles generated with
the standard \texttt{make\_moons} dataset generator;
\textbf{Circles}, consisting of two concentric circles generated with
\texttt{make\_circles} with factor $0.5$;
\textbf{XOR}, consisting of four Gaussian clusters centered at
$(\pm1,\pm1)$ with XOR labeling; and
\textbf{Helix}, consisting of two interleaved spirals with two full rotations,
parameterized as
$(t\cos t,t\sin t)$ and $(t\cos(t+\pi),t\sin(t+\pi))$
for $t\in[0,4\pi]$.

Each dataset contains 1{,}000 samples with additive Gaussian noise
($\sigma=0.1$). We use an 80/20 train/test split and standardize features to
zero mean and unit variance. In the experimental figures, class labels are
indexed from $0$ to $K-1$ to match the dataset and implementation conventions.

\subsubsection{Models}

We evaluate three versions of the fully separate neural realization
described in Section~\ref{subsec:neural_realization}, differing only in
the scalar activation function used to form the gate outputs.

\textbf{PUNN-Sigma.}
Each gate has the form
\[
q_i(x)=g_{\mathrm{sig}}(\theta_i(x)),
\]
where $\theta_i$ is a two-hidden-layer MLP with ReLU activations and
hidden-layer dimension $32$, giving $1{,}185$ parameters per gate.

\textbf{PUNN-Bump.}
The same neural parameterization is used with the compactly supported
bump activation
\[
g_{\mathrm{bump}}(t)
=
\begin{cases}
\exp\left(1-\dfrac{1}{1-t^2}\right),
& |t|<1,\\[1ex]
0,
& |t|\geq 1.
\end{cases}
\]
For Helix, the hidden-layer dimension is increased to $128$ to
accommodate the more complicated spiral geometry.

\textbf{PUNN-Gaussian.}
Each gate is given by
\[
q_i(x)
=
\exp\bigl(-\theta_i(x)^2\bigr),
\]
with the same gate-argument architecture as PUNN-Sigma and
$1{,}185$ parameters per gate.

\textbf{MLP Baseline.}
The baseline is a two-hidden-layer MLP with ReLU activations,
hidden-layer dimension $32$, and a softmax output, with $1{,}218$
trainable parameters.

Since all synthetic datasets are binary, PUNN uses a single gate,
\[
h_0(x)=q_0(x),
\qquad
h_1(x)=1-q_0(x).
\]
Thus, these experiments isolate the effect of the scalar activation
function while keeping the underlying neural parameterization largely
fixed.

\subsubsection{Training Protocol}

All models are trained for $200$ epochs using Adam
\cite{kingma2014adam}. The PUNN parameters are learned by minimizing
cross-entropy with respect to the partition probabilities, and the MLP
baseline is trained using the corresponding softmax cross-entropy loss.
The reported results are from a single training run with random seed $42$.

\subsubsection{Results}

Table~\ref{tab:synthetic} reports test accuracy for all methods. On these
four synthetic datasets, all PUNN variants achieve accuracy comparable to
the MLP baseline, with test accuracies ranging from $98.5\%$ to $100\%$.

\begin{table}[t]
\caption{Single-run test accuracy (\%) on the synthetic datasets
(random seed $42$). All PUNN variants achieve performance comparable to
the MLP baseline on these illustrative low-dimensional problems.}
\label{tab:synthetic}
\vskip 0.15in
\begin{center}
\begin{small}
\begin{sc}
\begin{tabular}{lcccc}
\toprule
Dataset & PUNN-Sigma & PUNN-Bump & PUNN-Gaussian & MLP \\
\midrule
Moons   & 100.0 & 100.0 & 100.0 & 100.0 \\
Circles & 99.0  & 99.0  & 98.5  & 98.5  \\
XOR     & 100.0 & 100.0 & 100.0 & 100.0 \\
Helix   & 100.0 & 99.5  & 98.5  & 99.5  \\
\midrule
Parameters & 1{,}185 & 1{,}186 & 1{,}185 & 1{,}218 \\
\bottomrule
\end{tabular}
\end{sc}
\end{small}
\end{center}
\vskip -0.1in
\end{table}

\subsubsection{Visualizing the Partition of Unity}

By construction, the PUNN outputs form a partition of unity:
\[
\sum_{i=0}^{K-1} h_i(x)=1
\]
for every input $x$. In the binary case considered here,
\[
h_0(x)+h_1(x)=1,
\]
so the two class-probability functions are complementary.

Figure~\ref{fig:synthetic_partition_functions} visualizes the partition
functions $h_0(x)$ and $h_1(x)$ learned by PUNN-Sigma on each dataset. Because the two partition functions are complementary, high values of one correspond to low values of the other, and their transition region determines the learned decision boundary. These visualizations illustrate how the learned partition of unity varies across the
feature space.

\begin{figure}[htbp]
\centering
\includegraphics[height=0.85\textheight,keepaspectratio]{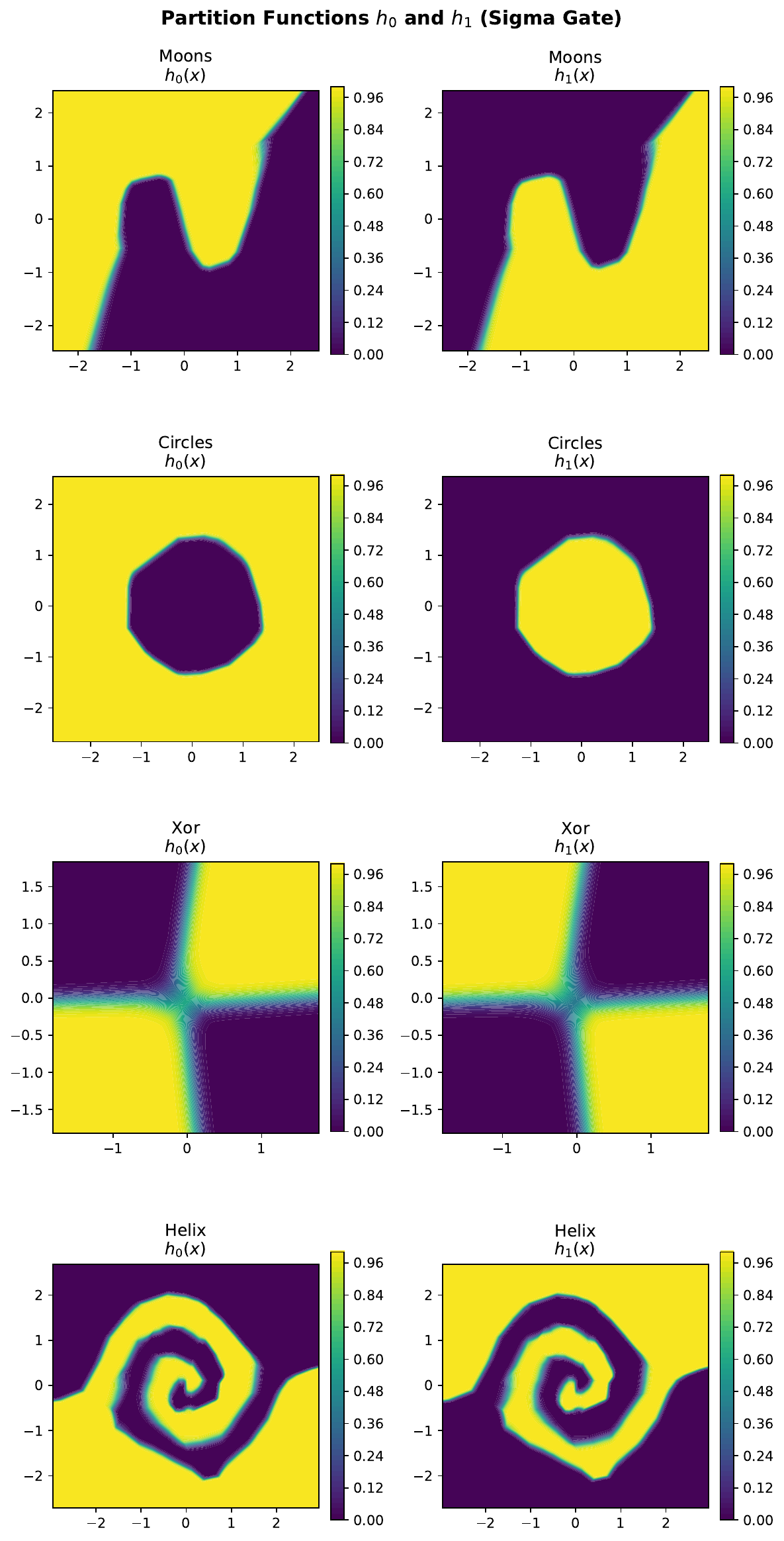}
\caption{Partition functions $h_0(x)$ and $h_1(x)$ learned by PUNN-Sigma.
Left column: $h_0(x)$; right column: $h_1(x)$. Rows correspond to Moons,
Circles, XOR, and Helix. The two functions satisfy
$h_0(x)+h_1(x)=1$ for every input $x$.}
\label{fig:synthetic_partition_functions}
\end{figure}

\FloatBarrier

\subsection{Real-World Datasets}
\label{sec:real_data}

We next evaluate PUNN on MNIST and CIFAR-100. These experiments examine
the partition-of-unity construction beyond the low-dimensional synthetic
examples and illustrate its behavior on standard multiclass
classification problems.

\subsubsection{MNIST Handwritten Digits}

\paragraph{Dataset.}
MNIST \cite{lecun1998mnist} consists of 70{,}000 grayscale images of
handwritten digits from classes $0$ through $9$, each of size
$28\times 28$ pixels. We use the standard split of 60{,}000 training images
and 10{,}000 test images. Each image is flattened into a
$784$-dimensional vector, and the input features are standardized using the
training-set mean and standard deviation.

\paragraph{Model.}
We use PUNN-Sigma for the MNIST experiments.

\textbf{PUNN-Sigma.}
Since MNIST has ten classes, the basic PUNN construction uses nine
sigmoid gate functions,
\[
q_i(x)=g_{\mathrm{sig}}(\theta_i(x)),
\qquad i=1,\ldots,9.
\]
For the gate-trace figures below, we relabel the indices from $0$ to $9$
to agree with the conventional MNIST digit labels.

Each gate argument $\theta_i$ is represented by a two-hidden-layer
feedforward neural network with ReLU activations. We denote by $H_g$
the dimension of each hidden layer; equivalently, the corresponding
weight matrices have dimensions
\[
W^{(g)}_1\in\mathbb R^{H_g\times 784},
\qquad
W^{(g)}_2\in\mathbb R^{H_g\times H_g},
\qquad
W^{(g)}_3\in\mathbb R^{1\times H_g},
\]
with bias vectors of the corresponding dimensions.

\paragraph{Training.}
The PUNN models are trained for $20$ epochs using Adam
\cite{kingma2014adam}, with learning rate $0.001$ and batch size $128$,
by minimizing the cross-entropy loss associated with the partition
functions $h_i$.

\medskip
\noindent\textbf{Interpretable Gate Trace.}
\medskip

Beyond classification accuracy, the ordered factorization of the PUNN
probabilities provides additional information about how an individual
prediction is formed. To illustrate this feature, we examine the gate
values and resulting partition probabilities for two representative
MNIST test images: one correctly classified but uncertain example and
one misclassified example.

\begin{figure}[htbp]
\centering
\includegraphics[width=\linewidth]{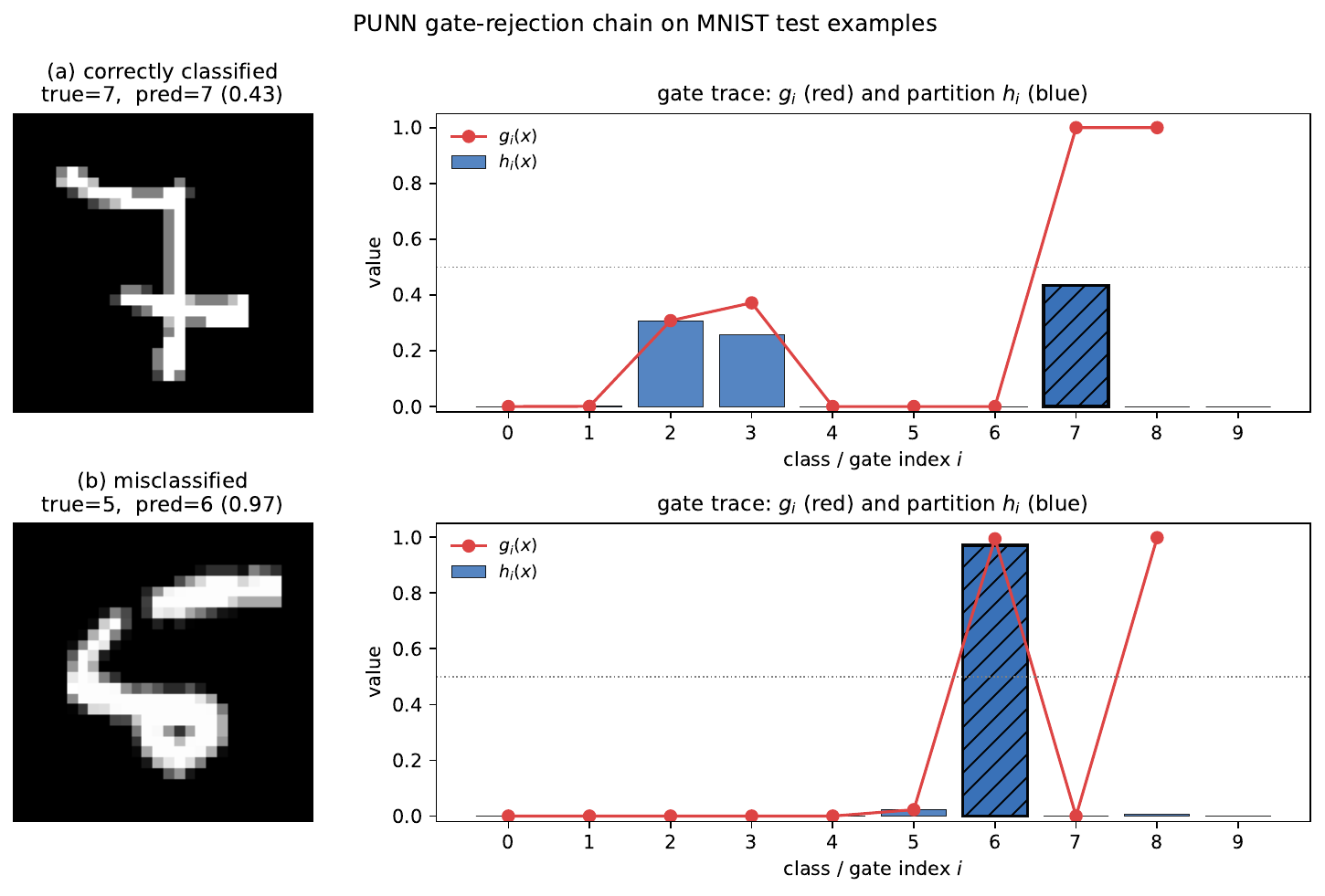}
\caption{PUNN gate-level diagnostic traces on two MNIST test examples.
\textbf{(a)}~A correctly classified but uncertain ``7'' (predicted with
confidence $0.43$). Gates~2 and~3 take intermediate values
($q_2=0.31$, $q_3=0.37$), producing class probabilities of $30.8\%$
and $25.7\%$ for classes~2 and~3, respectively, while the factors
associated with gate~7 produce a class~7 probability of $43\%$. The
trace exposes the model's ambiguity directly.
\textbf{(b)}~A misclassified example (true class~5, predicted~6 with
confidence $0.97$). The gate associated with the true class produces
only $q_5=0.023$, while the factors associated with class~6 produce
nearly all of the final probability. Red curve: gate values $q_i(x)$.
Blue bars: partition probabilities $h_i(x)$. Hatched bar: predicted
class.}
\label{fig:mnist_trace}
\end{figure}

Consider an uncertain test image of the digit ``7'' (Figure~\ref{fig:mnist_trace}a).
The ordered PUNN gate trace is
\[
q_0=0.00,\; q_1=0.00,\; q_2=\mathbf{0.31},\; q_3=\mathbf{0.37},\;
q_4{=}q_5{=}q_6=0.00,\; q_7=\mathbf{1.00},
\]
yielding partition probabilities
$h_2 = 0.31$, $h_3 = 0.26$, $h_7 = 0.43$ (others $\approx 0$).
The model's prediction is class~7, but the trace makes the residual ambiguity explicit:
gates~2 and~3 take intermediate values, indicating that the image
receives nonnegligible probability contributions from those classes,
while the complementary factors of the preceding gates and the value
$q_7=1$ produce the remaining probability assigned to class~7.
A misclassified example (Figure~\ref{fig:mnist_trace}b: true class~5, predicted~6) localizes the failure to a single gate:
$q_5 = 0.023$, so the gate associated with class~5 contributes very
little to the resulting class probability, while the resulting factors produce nearly all of the probability assigned to class~6.
Because PUNN learns the gate functions directly, the resulting trace is
a model-native representation of the factors forming the class
probabilities rather than a post-hoc decomposition applied after
prediction.

\medskip
\noindent\textbf{Sensitivity to class ordering.}
\label{sec:class_ordering}
\medskip

Because the recursive construction assigns classes to successive
partition positions, we examine whether predictive performance is
sensitive to the chosen class ordering. We evaluate PUNN-Sigma under
several different class orderings, with positions indexed from $0$ to $9$.

We evaluate ten class orderings, including the identity ordering
and nine randomly generated permutations. To separate ordering effects
from ordinary training variability, every ordering is trained using the
same three random seeds. Table~\ref{tab:mnist_random_ordering} reports
the resulting test accuracies.

\begin{table}[t]
\caption{MNIST test accuracy under ten class orderings. Each ordering is
evaluated using the same three random seeds. Entries report mean
$\pm$ standard deviation over the three seeds.}
\label{tab:mnist_random_ordering}
\vskip 0.1in
\begin{center}
\begin{small}
\begin{sc}
\begin{tabular}{lc}
\toprule
Class Ordering & Test Accuracy (\%) \\
\midrule
Identity      & $97.47 \pm 0.25$ \\
Permutation 1 & $97.41 \pm 0.13$ \\
Permutation 2 & $97.60 \pm 0.20$ \\
Permutation 3 & $97.45 \pm 0.11$ \\
Permutation 4 & $97.36 \pm 0.16$ \\
Permutation 5 & $97.30 \pm 0.15$ \\
Permutation 6 & $97.51 \pm 0.10$ \\
Permutation 7 & $97.26 \pm 0.10$ \\
Permutation 8 & $97.33 \pm 0.23$ \\
Permutation 9 & $97.40 \pm 0.11$ \\
\bottomrule
\end{tabular}
\end{sc}
\end{small}
\end{center}
\end{table}

The ordering-specific mean accuracies range from $97.26\%$ to
$97.60\%$. Their standard deviation is $0.10$ percentage points, and
their total range is $0.34$ percentage points. By comparison, the
average seed-to-seed standard deviation within a fixed ordering is
$0.15$ percentage points. Thus, under this experimental configuration, the effect of class ordering on test accuracy is small and comparable to the ordinary variability between training runs.

\FloatBarrier

\subsubsection{CIFAR-100}

We next consider CIFAR-100, which has $100$ classes. This experiment
illustrates the extension in Remark~\ref{rmk:multimodal}, where the
number of partition functions may exceed the number of classes.

We use the multi-output neural realization of
Section~\ref{subsec:neural_realization} with $2{,}000$ partition
functions, assigning $20$ partition functions to each class. The
probability of a class is obtained by summing the corresponding
partition functions. The resulting PUNN has $6{,}003{,}207$ trainable
parameters and is compared with a softmax CNN having a closely matched
parameter count of $5{,}887{,}971$.

Both models are trained on the standard CIFAR-100 training set and
evaluated on the standard test set. Results are reported over three
random seeds.

\begin{table}[H]
\caption{CIFAR-100 comparison over three random seeds. PUNN uses
$2{,}000$ partition functions, corresponding to $20$ partition
functions per class. Test accuracy is reported as mean $\pm$ standard
deviation.}
\label{tab:cifar100}
\vskip 0.1in
\begin{center}
\begin{small}
\begin{sc}
\begin{tabular}{lcc}
\toprule
Model & Parameters & Test Accuracy (\%) \\
\midrule
PUNN-Sigma
& $6{,}003{,}207$
& $70.65 \pm 0.22$ \\
Softmax CNN
& $5{,}887{,}971$
& $\mathbf{71.76 \pm 0.30}$ \\
\bottomrule
\end{tabular}
\end{sc}
\end{small}
\end{center}
\end{table}

As shown in Table~\ref{tab:cifar100}, PUNN achieves
$70.65\pm0.22\%$ test accuracy, compared with
$71.76\pm0.30\%$ for the closely parameter-matched softmax CNN,
a difference of $1.11$ percentage points. Thus, the recursive partition-of-unity construction remains effective in this $100$-class setting, achieving test accuracy close to that of a standard softmax classifier under a comparable parameter budget.

\FloatBarrier

\subsection{Shape-Informed Experiments}
\label{subsec:shape_informed_experiments}

We compare shape-informed PUNN with neural-network-parametrized baselines
on datasets for which geometric priors are plausible. The purpose of
these experiments is to examine the trade-off between predictive
accuracy and parameter efficiency when the geometry of the class regions
is encoded directly in the gate parametrization.

\begin{table}[t]
\caption{Comparison of shape-informed PUNN with a small MLP baseline
having two hidden layers of dimension $32$. Depending on the input
dimension and number of classes, the MLP contains approximately
$1.2$K--$1.3$K trainable parameters. The shape-informed models attain
comparable accuracy, and in some cases higher accuracy, with
parameter-count reductions ranging from $33\times$ to $304\times$.}
\label{tab:shape_informed}
\vskip 0.15in
\begin{center}
\begin{small}
\begin{sc}
\begin{tabular}{llccc}
\toprule
Dataset & Model & Accuracy (\%) & Parameters & Reduction \\
\midrule
Circles & Small MLP ($[32,32]$) & 98.6 & 1{,}218 & --- \\
Circles & Spherical shell & $\mathbf{98.9 \pm 0.58}$ & \textbf{4} &
\textbf{$304\times$} \\
\midrule
Moons & Small MLP ($[32,32]$) & 100.0 & 1{,}218 & --- \\
Moons & Fourier shell & 98.6 & \textbf{24} &
\textbf{$51\times$} \\
\midrule
Concentric rings 
& Small MLP ($[32,32]$) & 100.0 & 1{,}251 & --- \\
Concentric rings 
& Spherical shells & \textbf{100.0} & \textbf{8} &
\textbf{$156\times$} \\
\midrule
Iris & Small MLP ($[32,32]$) & 93.3 & 1{,}315 & --- \\
Iris & Harmonics  & \textbf{95.3} & 40 &
$33\times$ \\
Iris & Axis-aligned ellipsoid & \textbf{95.3} & \textbf{20} &
\textbf{$66\times$} \\
\bottomrule
\end{tabular}
\end{sc}
\end{small}
\end{center}
\vskip -0.1in
\end{table}

Table~\ref{tab:shape_informed} shows that geometry-informed
parameterizations can achieve accuracy comparable to the MLP baseline
with substantially fewer trainable parameters when the assumed geometry
is well matched to the class regions. The parameter reductions range
from $33\times$ to $304\times$ across the examples considered.

The Circles and Concentric Rings experiments illustrate the efficiency
of simple radial geometries, while the Moons and Iris experiments show
how more flexible Fourier, spherical-harmonic, and ellipsoidal
parameterizations can accommodate more complicated class regions.
Figure~\ref{fig:shape_partition_functions} visualizes the corresponding
partition functions for the Circles example.

\begin{figure}[htbp]
\begin{center}
\centerline{
\includegraphics[width=\textwidth]{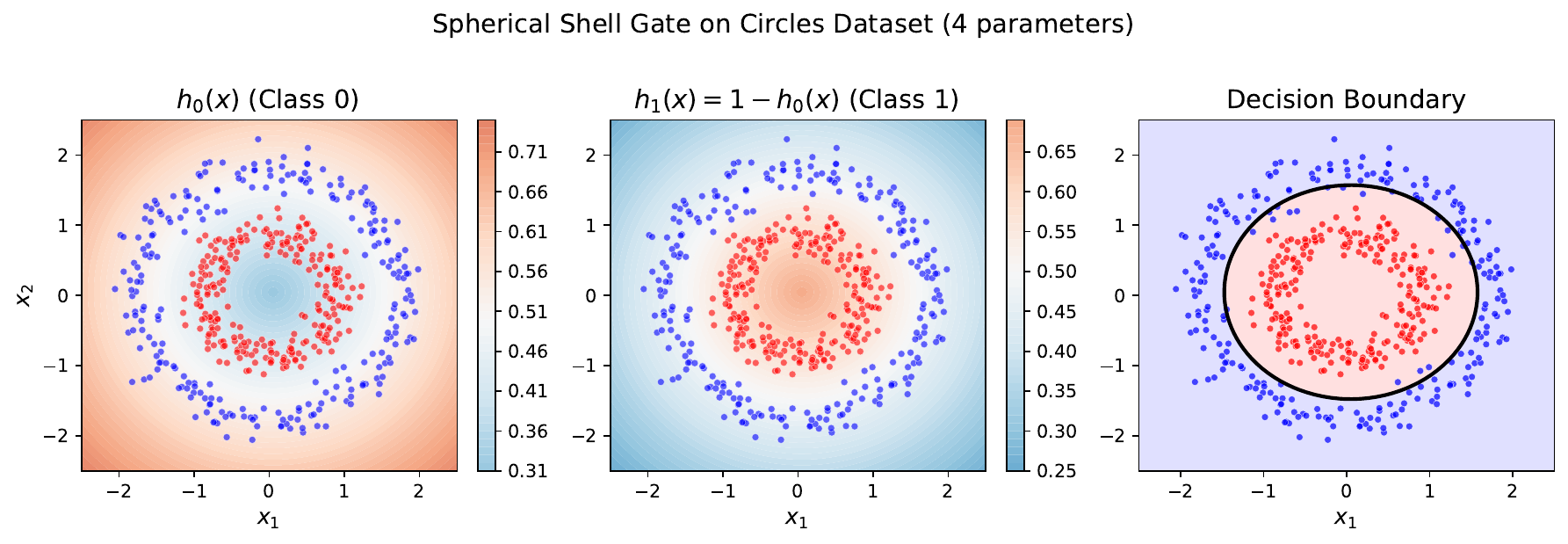}
}
\caption{Partition functions learned by the spherical-shell model on the
Circles dataset. Left: $h_0(x)$ assigns high probability to the inner
class. Middle: $h_1(x)=1-h_0(x)$ assigns high probability to the outer
class. Right: the decision boundary is given by $h_0(x)=0.5$. The two
partition functions satisfy $h_0(x)+h_1(x)=1$ for every input $x$.}
\label{fig:shape_partition_functions}
\end{center}
\vskip -0.2in
\end{figure}

\FloatBarrier

\subsubsection{Discussion}

The direction-dependent shell provides a flexible framework for encoding
geometric priors in PUNN. By varying the parameterization of
$r_{i,\mathrm{in}}(\hat n)$ and $r_{i,\mathrm{out}}(\hat n)$, one can
trade off geometric flexibility against model complexity. These
parameterizations are most natural when the class regions have known or
approximately known geometric structure and the feature dimension is
low to moderate.

In higher dimensions, the number of coefficients required to represent
general direction-dependent radii may grow rapidly. PUNN therefore also
allows hybrid constructions in which geometry-informed gates are used
for classes with suitable structure and neural-network gates are used
for more irregular class regions.

\section*{Acknowledgments}

The author thanks Soheil Kolouri for valuable discussions and feedback
on this work, and in particular for suggesting the shared multi-output
realization of the PUNN gates.

\noindent\textbf{Use of AI tools.} AI-assisted tools were used during the preparation of this work. Devin was used as a coding assistant for the implementation and verification of the numerical experiments in Section~\ref{sec:experiments}. Generative AI tools, including ChatGPT, were used to assist with editing and organization of the manuscript. The mathematical content and scientific conclusions are those of the author.

\section*{Declarations}

\noindent\textbf{Funding.} A. Aldroubi was supported in part by the National Institute of General Medical Sciences of the National Institutes of Health under Award Number R01GM130825, through a subaward from the University of Virginia.

\noindent\textbf{Competing interests.} The author declares no competing interests.

\noindent\textbf{Data availability.} The MNIST and CIFAR-100 datasets used in Section~\ref{sec:experiments} are standard, publicly available benchmark datasets. The synthetic datasets are generated according to the procedures described in Section~\ref{sec:synthetic} and can be regenerated from that description. Code implementing the PUNN architecture and the experiments reported in this paper is available from the author upon reasonable request and will be made publicly available upon acceptance.

\noindent\textbf{Author contributions.} A. Aldroubi is the sole author of this manuscript and is responsible for all aspects of the work.

\noindent\textbf{Ethics approval.} Not applicable. This study did not involve human participants, human data, or animals.

\bibliographystyle{amsplain}
\bibliography{PUNN}
\end{document}